\documentclass[twocolumn, 8pt]{extarticle}
\usepackage[margin=0.75in]{geometry}
\usepackage{glossaries, amssymb, mathtools, amsthm, algorithm, algpseudocode, authblk, hyperref}
\usepackage[numbers, sort&compress]{natbib}
\glsdisablehyper
\newacronym{MDP}{MDP}{Markov decision process}
\newacronym{RL}{RL}{reinforcement learning}
\newacronym{DRL}{DRL}{deep reinforcement learning}
\newacronym{SAC}{SAC}{Soft Actor-Critic}
\newacronym{DQN}{DQN}{Deep Q-Network}
\newacronym{NNs}{NNs}{neural networks}
\newacronym{KL-divergence}{KL-divergence}{Kullback-Leibler divergence}
\newacronym{CO}{CO}{combinatorial optimization}
\newacronym{CVaR}{CVaR}{conditional value at risk}
\newacronym{RHS}{RHS}{right-hand side}

\newtheorem{proposition}{Proposition}

\title{\Huge{Risk-Sensitive Soft Actor-Critic for Robust Deep Reinforcement Learning under Distribution Shifts}}

\author{\Large{Tobias Enders}}
\affil{\normalsize{Technical University of Munich, Munich, Germany, \texttt{tobias.enders@tum.de}}}
\author{James Harrison}
\affil{Google DeepMind, San Francisco, CA, USA, \texttt{jamesharrison@google.com}}
\author{Maximilian Schiffer}
\affil{Technical University of Munich, Munich, Germany, \texttt{schiffer@tum.de}}

\date{}

\begin{document}

\twocolumn[
  \begin{@twocolumnfalse}
  
\maketitle

\begin{abstract}
\normalsize
\noindent
We study the robustness of deep reinforcement learning algorithms against distribution shifts within contextual multi-stage stochastic combinatorial optimization problems from the operations research domain. In this context, risk-sensitive algorithms promise to learn robust policies. While this field is of general interest to the reinforcement learning community, most studies up-to-date focus on theoretical results rather than real-world performance. With this work, we aim to bridge this gap by formally deriving a novel risk-sensitive deep reinforcement learning algorithm while providing numerical evidence for its efficacy. Specifically, we introduce discrete Soft Actor-Critic for the entropic risk measure by deriving a version of the Bellman equation for the respective $Q$-values. We establish a corresponding policy improvement result and infer a practical algorithm. We introduce an environment that represents typical contextual multi-stage stochastic combinatorial optimization problems and perform numerical experiments to empirically validate our algorithm's robustness against realistic distribution shifts, without compromising performance on the training distribution. We show that our algorithm is superior to risk-neutral Soft Actor-Critic as well as to two benchmark approaches for robust deep reinforcement learning. Thereby, we provide the first structured analysis on the robustness of reinforcement learning under distribution shifts in the realm of contextual multi-stage stochastic combinatorial optimization problems. \\
\newline
\textbf{Keywords:} deep reinforcement learning, robustness, risk-sensitivity, distribution shifts, combinatorial optimization problems
\end{abstract}

\vspace{0.5in}

  \end{@twocolumnfalse}
]


\section{Introduction}

Model-free \gls{DRL} has recently been used increasingly to solve contextual multi-stage stochastic \gls{CO} problems from the operations research domain. Such problems arise, among others, in inventory control \cite{Vanvuchelen2020, Gijsbrechts2022, DeMoor2022} or control of mobility on demand systems \cite{Tang2019, Gammelli2021, SadeghiEshkevari2022, Enders2023}, where an agent must solve a \gls{CO} problem at each time step to decide its policy. In this context, \gls{DRL} benefits from using \gls{NNs}, which are well suited to represent and learn complex policies in such contextual environments. However, a policy trained by \gls{DRL} to solve a respective \gls{MDP} is sensitive to changes of environment parameters \cite{Iyengar2005}. Yet, the resulting question of how to improve \gls{DRL}'s robustness against such disturbances has received limited attention in the literature on \gls{DRL} for multi-stage stochastic \gls{CO} so far. With this work, we take a first step to close this research gap.

While works on robust \gls{RL} exist, they are often tailored to robotics applications and typically focus on topics such as safety during exploration, robustness against action perturbations, or robustness against adversarial attacks. In contrast, robustness against distribution shifts is particularly relevant for contextual multi-stage stochastic \gls{CO} problems. Such distribution shifts can occur because of unforeseeable changes in the environment, e.g., altered transportation patterns after the surge of Covid or disruptions of supply chains triggered by geopolitical events. Besides, imperfect simulators can lead to distribution shifts when training in simulation before deployment in the real world. Additionally, continuous action spaces are prevalent in robotics, while we often encounter discrete action spaces in \gls{CO} problems. Furthermore, existing works on robust \gls{RL} often focus on finding the optimal policy for a worst-case scenario. Instead, in a \gls{CO} context, one focuses on the tradeoff between learning a policy that consistently achieves high expected returns on the training distribution and a policy that is robust against distribution shifts at the price of lower yet good performance across distributions. In the following, we develop a methodology that provides a principled approach to explicitly control this consistency-robustness tradeoff.

\subsection{Related work} \label{sec:literature}

Multiple approaches to improve the robustness of \gls{RL} exist, see the extensive review in \cite{Moos2022}. However, most of them have limited applicability in practice, since they are not scalable with \gls{NNs} as function approximators or require lots of additional machinery, e.g., solving a bi-level optimization problem or introducing adversaries. Alternative approaches are model-based or do not give control over the consistency-robustness tradeoff. In contrast, we focus on efficient and tractable approaches that can build on state-of-the-art model-free \gls{DRL} algorithms to achieve robustness under realistic distribution shifts. Some approaches can fulfill these requirements: manipulation of training data, entropy regularization, and risk-sensitive \gls{RL}. We use the former two approaches as benchmarks and focus our study on risk-sensitive \gls{RL}. Thus, we defer the discussion of further details on the benchmarks to the experiments and results sections and focus the remaining discussion on risk-sensitive \gls{RL}. 

In risk-sensitive \gls{RL}, one usually optimizes a risk measure, e.g., mean-variance or \gls{CVaR}, of the return instead of the expected return. Theoretical connections exist between risk-sensitive and robust \gls{RL} when optimizing the \gls{CVaR} \cite{Chow2015} as well as for coherent risk measures and risk-sensitivity under expected exponential utility functions \cite{Osogami2012}. Other works report promising empirical results for risk-sensitive \gls{RL} algorithms, see, e.g., \cite{Zhang2021, Noorani2022}. Moreover, a risk-sensitive \gls{RL} approach can lead to algorithms similar to risk-neutral \gls{RL}, thus requiring little extra implementation effort. Depending on the risk measure, risk-sensitive \gls{RL} algorithms can have a hyperparameter controlling the tradeoff between expected return and risk-sensitivity. 

Some works consider a constrained optimization problem, where the expected return is maximized while constraining a risk measure of the returns to set a threshold for unwanted outcomes, see, e.g., \cite{Prashanth2013, Prashanth2014, Prashanth2016, Yang2021}. This approach has a strong focus on safety and/or worst-case outcomes and thus fits our purposes less than directly optimizing a risk-sensitive objective function. An approach to optimize risk-sensitive objective functions is distributional \gls{RL}, see, e.g., \cite{Ma2020, Singh2020, Urpi2021}. It has the advantage that the learned distributional information can be used to consider any risk measure, i.e., the proposed algorithms are agnostic to the choice of risk measure. However, this comes at the cost of additional algorithmic complexity.

Thus, we focus on approaches that directly optimize a risk-sensitive objective. To do so, \cite{Howard1972, Jacquette1976, Patek2001} use dynamic programming, assuming knowledge of the underlying \gls{MDP}. The authors of \cite{Tamar2012, Chow2014, Chow2015} propose policy gradient and/or actor-critic and/or approximate value iteration algorithms for objective functions based on variance or \gls{CVaR}. In \cite{Tamar2015}, a policy gradient and actor-critic algorithm is derived for the whole class of coherent risk measures \cite{Artzner1999}. All of these works date back to the year 2015 or earlier and develop basic \gls{RL} algorithms, rather than building on today's state-of-the-art. In particular, they do not include \gls{NNs} as function approximators to deal with large state spaces. The reported experiments focus on robustness under action perturbations or risk-sensitivity in finance applications without a relation to robustness or a \gls{CO} context. 

Recently, exponential criteria and the entropic risk measure have been increasingly used as a risk-sensitive objective, with \cite{Fei2020, Fei2021, Fei2021a} focusing on theoretical regret analysis rather than practical state-of-the-art algorithms or experimental results. The authors of \cite{Nass2019} derive a policy gradient algorithm for the entropic risk measure and apply it in a robotics context. In \cite{Noorani2021}, a risk-sensitive variant of the REINFORCE algorithm for exponential criteria is introduced and tested on the Cart Pole and Acrobot environments. The risk-sensitive algorithm outperforms its risk-neutral counterpart, even though the environment at test time is the same as during training (no disturbance). This work is extended to an actor-critic algorithm, which outperforms its risk-neutral counterpart on Cart Pole and Acrobot environments when they are disturbed by varying the pole lengths during testing \cite{Noorani2022}. Although these works use \gls{NNs} for function approximation, they are still the risk-sensitive counterparts to basic instead of state-of-the-art risk-neutral \gls{RL} algorithms. This is not the case for \cite{Zhang2021}, which develops a risk-sensitive version of TD3 for a mean-variance objective and evaluates its performance on MuJoCo environments with disturbed actions. However, this algorithm is not compatible with a discrete action space, which is the focus of our study in a \gls{CO} context.

Concluding, to the best of our knowledge, no model-free risk-sensitive \gls{DRL} algorithm for discrete actions that is based on the state-of-the-art in risk-neutral \gls{DRL} exists. Moreover, none of the existing works investigates robustness against distribution shifts, and most works on risk-sensitive \gls{RL} compare the performance of risk-sensitive \gls{RL} algorithms only to their risk-neutral counterparts. Finally, there exists no published work with a structured analysis of the robustness of \gls{RL} in \gls{CO} problems. 

\subsection{Contributions}

We aim to close the research gap outlined above by introducing a novel risk-sensitive \gls{DRL} algorithm: discrete \gls{SAC} for the entropic risk measure, which effectively learns policies that exhibit robustness against distribution shifts. Specifically, we derive a version of the Bellman equation for $Q$-values for the entropic risk measure. We establish a corresponding policy improvement result and infer a practical model-free, off-policy algorithm that learns from single trajectories. From an implementation perspective, our algorithm requires only a small modification relative to risk-neutral \gls{SAC} and is therefore easily applicable in practice. Furthermore, our algorithm allows to control the consistency-robustness tradeoff through a hyperparameter. For empirical evaluation, we propose a grid world environment that abstracts multiple relevant contextual multi-stage stochastic \gls{CO} problems. 

We show that our algorithm improves robustness against distribution shifts without performance loss on the training distribution compared to risk-neutral \gls{SAC}. Moreover, we evaluate our algorithm in comparison to two other practically viable approaches to achieve robustness: manipulating the training data and entropy regularization. The performance analysis of these approaches within our environment under distribution shifts is of independent interest. While manipulating the training data leads to good empirical results, it is less generally applicable than our risk-sensitive algorithm and entropy regularization. Entropy regularization achieves better robustness but worse performance on the training distribution compared to our risk-sensitive algorithm. To facilitate a direct comparison of the two approaches, we study the weighted average of the performance on the training distribution and the performance under distribution shifts: our risk-sensitive algorithm outperforms entropy regularization if we assign at least 37\% weight to the performance on the training distribution. Overall, we provide the first structured analysis of the robustness of \gls{RL} under distribution shifts in a \gls{CO} context. To foster future research and ensure reproducibility, our code is publicly available at \url{https://github.com/tumBAIS/RiskSensitiveSACforRobustDRLunderDistShifts}.

\section{Risk-sensitive Soft Actor-Critic for discrete actions} \label{sec:methodology}

We base our novel risk-sensitive \gls{DRL} algorithm on the variant of \gls{SAC} \cite{Haarnoja2018} for discrete actions \cite{Christodoulou2019}. We chose \gls{SAC} for three reasons: firstly, it is a state-of-the-art algorithm with very good performance across a wide range of environments \cite{Haarnoja2018c, Christodoulou2019, Iqbal2019, Wong2021, Sun2022, Enders2023, Yang2023, Hu2023}; secondly, it is an off-policy algorithm and consequently more sample-efficient than on-policy algorithms; thirdly, it already incorporates entropy regularization. The last reason makes \gls{SAC} particularly well suited for our robustness analysis: entropy regularization can increase robustness, as shown theoretically in \cite{Eysenbach2022} and empirically in \cite{Haarnoja2018a, Haarnoja2018b, Eysenbach2022}. Consequently, we can benchmark the robustness of our risk-sensitive algorithm with turned off entropy regularization against the robustness of risk-neutral \gls{SAC} with different intensities of entropy regularization. 

In the remainder, we use the following basic notation: $t$ denotes a time step, $s$ a state, $a$ an action, $r$ a reward, $d$ a done signal, $\gamma$ the discount factor, and $s'$ the next state if there is no time index when considering a single transition. We denote a (stochastic) policy by $\pi$ and its entropy given state $s$ by $\mathcal H\left(\pi\left(\cdot\lvert s\right)\right)$. Moreover, $\alpha\geq0$ is the entropy coefficient hyperparameter, $\Pi$ denotes the space of tractable policies and $Q^\pi$ the $Q$-values under policy~$\pi$. The notation $\left(s_{t+1},...\right)\sim\rho_\pi$ refers to sampling a state-action trajectory, i.e., $s_{t+1}\sim p(s_{t+1}\lvert s_t,a_t),\, a_{t+1}\sim\pi(a_{t+1}\lvert s_{t+1}),\, s_{t+2}\sim p(s_{t+2}\lvert s_{t+1},a_{t+1}),\, ...\, ,\, s_T\sim p(s_T\lvert s_{T-1},a_{T-1})$, where $p$ denotes the state transition probability distribution and $T\in\mathbb N\cup\{\infty\}$ the terminal time step. Furthermore, $D_\text{KL}(p\lvert\lvert q)$ is the \gls{KL-divergence} of probability distribution $q$ from probability distribution $p$, while $D$ is the replay buffer. When considering a parameterized policy $\pi_\phi$ or parameterized $Q$-values $Q_\theta$, we denote the actor network parameters by $\phi$ and critic network parameters by $\theta$, as well as target critic network parameters by $\overline\theta$. When we write $\pi(s)$ instead of $\pi(a\lvert s)$ or $Q(s)$ instead of $Q(s,a)$, we refer to the vector of all action probabilities or the vector of $Q$-values for all actions, respectively, given state $s$ (as opposed to the single entry of this vector for the specific action $a$).

\subsection{Risk-neutral Soft Actor-Critic} \label{sec:risk_neutral_SAC}

We provide a short summary of the risk-neutral \gls{SAC} algorithm for discrete actions, before deriving our risk-sensitive version. \gls{SAC} is an off-policy \gls{RL} algorithm which concurrently trains an actor network that parameterizes a stochastic policy, i.e., a probability distribution over all possible actions, and a critic network that outputs the $Q$-values for all possible actions, given an input state. It regularizes rewards with an additional entropy term to explicitly incentivize exploration, such that the optimization objective reads
\begin{equation*}
    \max_\pi \mathbb E_\pi \left[ \sum_{t=0}^\infty \gamma^t\left(r\left(s_t,a_t\right) + \alpha\mathcal H\left(\pi\left(\cdot\lvert s_t\right)\right)\right) \right] .
\end{equation*}
The entropy coefficient $\alpha$ controls the trade-off between rewards and the entropy. It can be set as a hyperparameter or learned such that the resulting policy has a certain target entropy chosen based on some heuristic. In the remainder, we use the former option. 

In the discrete actions setting, the loss functions for the actor and the critic read
\begin{align}
    J_\pi(\phi) &= \mathbb E_{s\sim D}\left[ \pi_\phi(s)^T\cdot\left( \alpha\log \pi_\phi(s) - Q_\theta(s) \right) \right], \label{eq:policy_loss} \\
    J_Q(\theta) &= \mathbb E_{(s,a,r,d,s')\sim D}\left[\frac12\left( Q_\theta(s,a) - \hat Q \right)^2\right], \nonumber \\
    \hat Q &= r + (1-d)\gamma\cdot\pi_\phi(s')^T\cdot\left(Q_{\overline\theta} \left(s'\right) - \alpha\log\pi_\phi(s')\right). \nonumber
\end{align}
In practice, we train two (target) critic networks and use the minimum of the two $Q$-values in the policy loss and $\hat Q$ calculation, to mitigate the overestimation bias. We use this minimum of two critics for our risk-sensitive algorithm analogously but do not explicitly write the minimum of two $Q$-values for conciseness.

In the following, we firstly introduce our risk-sensitive objective function. Secondly, we derive a Bellman equation for the $Q$-values under this objective. Thirdly, we show how to achieve policy improvement. Finally, we derive a practical algorithm with function approximation based on these theoretical results.

\subsection{Risk-sensitive objective} \label{sec:objective}

Instead of the risk-neutral objective function, i.e., the expected value of the sum of discounted rewards, we use the entropic risk measure \cite{Howard1972, Jacquette1976, Jacobson1973, Whittle1981} as our risk-sensitive objective:
\begin{equation}
    \max_\pi \frac{1}{\beta} \log \mathbb E_\pi \left[ e^{\beta \cdot \sum_{t=0}^\infty \gamma^t\left(r\left(s_t,a_t\right) + \alpha\mathcal H\left(\pi\left(\cdot\lvert s_t\right)\right)\right)} \right]. \label{eq:objective}
\end{equation}
Here, we adjust the standard entropic risk measure by introducing entropy regularization of rewards to obtain a risk-sensitive variant of \gls{SAC}. Analogously to risk-neutral \gls{SAC}, entropy regularization allows us to explicitly control the exploration-exploitation tradeoff in our risk-sensitive algorithm via the entropy coefficient. The hyperparameter $\beta\in\mathbb R$ controls the risk-sensitivity of the objective:
\begin{equation}
     \frac{1}{\beta} \log \mathbb E \left[ e^{\beta R} \right] = \mathbb E[R] + \frac{\beta}{2}\mathrm{Var}[R] + \mathcal O\left(\beta^2\right) , \label{eq:taylor_expansion}
\end{equation}
where we set $R = \sum_{t=0}^\infty \gamma^t \cdot r\left(s_t,a_t\right)$. Since the variance measures uncertainty, we obtain a risk-averse objective function for $\beta<0$, for $\beta>0$ it is risk-seeking, and $\beta\rightarrow0$ recovers the common risk-neutral objective.

In principle, we could use any risk measure. However, the entropic risk measure is a natural choice, particularly for operations research applications, as it is the certainty-equivalent expectation of the exponential utility function, see \cite{Howard1972, Jacquette1976}. Risk-sensitive \gls{MDP}s with the expected exponential utility function as optimization objective are closely connected to robust \gls{MDP}s \cite{Osogami2012}. Also, the entropic risk measure is a convex risk measure, thus fulfilling multiple desirable properties from a mathematical risk management perspective, even though it is not coherent \cite{Artzner1999}. Furthermore, alternative measures such as \gls{CVaR} are computed solely based on the tail of the return distribution, using only a small portion of the data. Contrarily, the entropic risk measure bases on all data, which is important in an \gls{RL} context, where sample efficiency is a major concern. Besides, \gls{CVaR} does not give explicit control over the tradeoff between expected return and risk, as opposed to the entropic risk measure: Equation~\eqref{eq:taylor_expansion} reveals that the entropic risk measure allows to explicitly control this tradeoff by setting $\beta$ accordingly.

Moreover, we will show below that the entropic risk measure is well suited to derive a practical \gls{RL} algorithm without adding a lot of machinery. The resulting algorithm can still leverage a lot of valuable techniques developed for risk-neutral \gls{RL}, like off-policy learning with experience replay and entropy regularization for effective exploration. 

Since we theoretically expect risk-averse rather than risk-seeking behavior to improve robustness, we use negative values for $\beta$ in the empirical evaluation. Nevertheless, the theoretical results which we derive in the following as well as the practical algorithm also apply to the risk-seeking case with $\beta>0$.

In the following, we use two approximations (see Appendix~\ref{app:approximations} for details): (i) for $\beta$ close to zero and a real-valued random variable $X$, we get $\mathbb E\left[e^{\beta X}\right] \approx e^{\beta\, \mathbb E\left[X\right]}$; (ii) for $\gamma$ close to one, it holds that $\mathbb E\left[X^\gamma\right] \approx \left(\mathbb E[X]\right)^\gamma$.

\subsection{Bellman equation}
To facilitate value iteration, we derive a Bellman equation for the $Q$-values under our risk-sensitive objective. We want to use the Bellman equation as the basis for a model-free \gls{RL} algorithm that can learn from single trajectories. Thus, the Bellman equation should take the form $Q^\pi\!\! \left(s_t,a_t\right) = \mathbb E_{\left(s_{t+1},a_{t+1}\right)\sim\rho_\pi} \left[ \cdot \right]$, which allows to obtain an unbiased sampling-based estimate of the \gls{RHS}. With the logarithm in Equation~\eqref{eq:objective}, we cannot obtain a Bellman equation of that form. Thus, we define
\begin{equation}
    \overline Q^\pi\!\! (s_t,a_t) := e^{\beta \cdot Q^\pi\!\! (s_t,a_t)} \iff Q^\pi\!\! \left(s_t,a_t\right) = \frac{1}{\beta} \log \overline Q^\pi\!\!\left(s_t,a_t\right) \label{eq:Qbar}
\end{equation}
and derive a Bellman equation for $\overline Q$. While we use the same rationale as \cite{Noorani2022} here, we cannot use this work's result, as it considers $V$-values and does not incorporate entropy regularization. In the following, we use the notation $\overline Q^\pi_t := \overline Q^\pi\!\!\left(s_t,a_t\right)$, $r_t:=r\left(s_t,a_t\right)$, and $\mathcal H^\pi_t := \mathcal H\left(\pi\left(\cdot\lvert s_t\right)\right)$ to save space.

\begin{proposition}[Bellman equation]
    For the risk-sensitive objective in Equation~\eqref{eq:objective}, $\gamma$ close to one, and with $\overline Q$ as defined in Equation~\eqref{eq:Qbar}, it holds that 
    \begin{equation}
        \overline Q^\pi_t = \mathbb E_{\left(s_{t+1},a_{t+1}\right)\sim\rho_\pi} \left[ \exp\left( \beta r_t + \beta\gamma\alpha\mathcal H^\pi_{t+1} + \gamma\log\overline Q^\pi_{t+1} \right)\right] . \label{eq:bellman}
    \end{equation}
\end{proposition}

\begin{proof}
    With the definition of $Q$-values, we obtain 
    \footnotesize
    \begin{align*}
        \overline Q^\pi_t &= e^{\beta r_t}\cdot \mathbb E_{\left(s_{t+1},...\right)\sim\rho_\pi} \left[ \exp\left( \beta\cdot\sum_{l=1}^\infty \gamma^l \left( r_{t+l} + \alpha\mathcal H^\pi_{t+l} \right) \right) \right] \nonumber \\
        &\stackrel{\mathmakebox[\widthof{=}]{\text{(a)}}}{=} e^{\beta r_t}\cdot \mathbb E_{\left(s_{t+1},a_{t+1}\right)\sim\rho_\pi} \Biggl[ \exp\left(\beta\gamma\left( r_{t+1}+\alpha\mathcal H^\pi_{t+1}\right)\right)\Biggr. \nonumber \\
        & \hspace{1.9cm} \Biggl. \cdot \mathbb E_{\left(s_{t+2},...\right)\sim\rho_\pi} \left[ \exp\left( \beta\cdot\sum_{l=2}^\infty \gamma^l \left( r_{t+l} + \alpha\mathcal H^\pi_{t+l} \right) \right) \right]\Biggr] \nonumber \\
        &\stackrel{\mathmakebox[\widthof{=}]{\text{(b)}}}{=} e^{\beta r_t}\cdot \mathbb E_{\left(s_{t+1},a_{t+1}\right)\sim\rho_\pi} \Biggl[ \exp\left(\beta\gamma\alpha\mathcal H^\pi_{t+1}\right) \Biggr. \nonumber \\
        & \hspace{-0.03cm} \Biggl. \cdot \left( e^{\beta r_{t+1}} \cdot \mathbb E_{\left(s_{t+2},...\right)\sim\rho_\pi} \left[ \exp\left( \beta\cdot\sum_{l=2}^\infty \gamma^{l-1} \left( r_{t+l} + \alpha\mathcal H^\pi_{t+l} \right) \right) \right]\right)^\gamma\Biggr] \nonumber \\
        &\stackrel{\mathmakebox[\widthof{=}]{\text{(c)}}}{=} e^{\beta r_t}\cdot \mathbb E_{\left(s_{t+1},a_{t+1}\right)\sim\rho_\pi} \left[ \exp\left(\beta\gamma\alpha\mathcal H^\pi_{t+1}\right) \cdot \left( \overline Q^\pi_{t+1}\right)^\gamma\right] \nonumber \\
        &= \mathbb E_{\left(s_{t+1},a_{t+1}\right)\sim\rho_\pi} \left[ \exp\left( \beta r_t + \beta\gamma\alpha\mathcal H^\pi_{t+1} + \gamma\log\overline Q^\pi_{t+1} \right)\right].
    \end{align*}
    \normalsize
    Equality (a) is based on the observation that the first term in the sum does not depend on the transitions after $t+1$. Equality (b) follows from Approximation (ii). Equality (c) uses the definition of $Q$-values for time step $t+1$. 
\end{proof}

\subsection{Policy improvement}
Given the $Q$-values obtained under an old policy $\pi_\text{old}$, we obtain a new policy as
\begin{align}
    \pi_\text{new} & = \arg\min_{\pi'\in\Pi} D_\text{KL} \left( \pi'\left(\cdot\lvert s_t\right) \Bigg\lvert\Bigg\lvert \frac{e^{\frac{1}{\alpha}Q^{\pi_\text{old}}\!\left(s_t,\cdot\right)}}{Z^{\pi_\text{old}}\left(s_t\right)}\right) \nonumber \\
    & = \arg\min_{\pi'\in\Pi} J_{\pi_\text{old}}\left(\pi'\left(\cdot\lvert s_t\right)\right) , \label{eq:new_policy}
\end{align}
exactly as in the original \gls{SAC} paper \cite{Haarnoja2018}. Here, the partition function $Z^{\pi_\text{old}}\left(s_t\right)$ normalizes the distribution. We show that despite our changed objective function and thus different Bellman equation, this definition of a new policy still implies policy improvement. 

\begin{proposition}[Policy improvement] \label{prop:policy_improvement}
    For an old policy $\pi_\text{old}$ and the new policy $\pi_\text{new}$ as defined in Equation~\eqref{eq:new_policy}, it holds that $Q^{\pi_\text{new}}(s_t,a_t) \geq Q^{\pi_\text{old}}(s_t,a_t)$ for any state $s_t$ and action $a_t$, assuming that $\gamma$ is close to one and $\beta$ is close to zero.
\end{proposition}

\begin{proof}
    We sketch the proof here and provide details in Appendix~\ref{app:proof_policy_improvement}. 
    
    We can always choose $\pi_\text{new} = \pi_\text{old}$, such that $J_{\pi_\text{old}}\left(\pi_\text{new}\left(\cdot\lvert s_t\right)\right) \leq J_{\pi_\text{old}}\left(\pi_\text{old}\left(\cdot\lvert s_t\right)\right)$, which yields
    \begin{align*}
        &\mathbb E_{a_t\sim\pi_\text{new}} \left[Q^{\pi_\text{old}}_t-\alpha\log\pi_\text{new}\left(a_t\lvert s_t\right)\right] \nonumber \\
        &\geq \mathbb E_{a_t\sim\pi_\text{old}} \left[Q^{\pi_\text{old}}_t-\alpha\log\pi_\text{old}\left(a_t\lvert s_t\right)\right] .
    \end{align*}
    For $\beta<0$, this is equivalent to
    \begin{equation*}
        e^{\beta\alpha\mathcal H^{\pi_\text{new}}_t}\cdot\mathbb E_{a_t\sim\pi_\text{new}} \left[\overline Q^{\pi_\text{old}}_t\right] \leq e^{\beta\alpha\mathcal H^{\pi_\text{old}}_t}\cdot\mathbb E_{a_t\sim\pi_\text{old}} \left[\overline Q^{\pi_\text{old}}_t\right] .
    \end{equation*}
    Repeated application of the Bellman equation derived before and this inequality gives $\overline Q^{\pi_\text{old}}_t \geq \overline Q^{\pi_\text{new}}_t$, such that we obtain $Q^{\pi_\text{old}}_t \leq Q^{\pi_\text{new}}_t$. The proof for $\beta>0$ works analogously. 
\end{proof}

\subsection{Practical algorithm}
We can use the Bellman equation and the policy improvement result in a straightforward manner to obtain a practical risk-sensitive off-policy \gls{DRL} algorithm. It is similar to risk-neutral \gls{SAC}, with the following adjustments:

We can learn $\overline Q$ with the corresponding critic loss function
\begin{align*}
    J_{\overline Q}(\theta) & =\mathbb E_{(s,a,r,d,s')\sim D}\left[\frac12\left( \overline Q_\theta(s,a) - \hat Q \right)^2\right], \nonumber \\
    \hat Q & = \exp\left(\beta r-\beta\gamma\alpha\: \pi_\phi(s')^T\log\pi_\phi(s') \right) \cdot \pi_\phi(s')^T \left(\overline Q_{\overline\theta} \left(s'\right)\right)^\gamma ,
\end{align*}
where the target $\hat Q$ is an unbiased, sampling-based estimate of Equation~\eqref{eq:bellman}. We computed the entropy and the expectation over next actions directly, which is possible because of the discrete action space, such that we only sample the next state to estimate Equation~\eqref{eq:bellman}. If the done signal $d$ is $\mathrm{TRUE}$, we set the terms after $e^{\beta r}$ to one. To ensure that $\overline Q>0$, we use a softplus instead of a linear activation on the output layer of the critic networks.

Based on Equation~\eqref{eq:new_policy}, the policy loss function is the same as in risk-neutral \gls{SAC}, shown in Equation~\eqref{eq:policy_loss}, but we calculate the needed critic values as $Q_\theta (s) = \frac{1}{\beta} \log \overline Q_\theta(s)$ from the output $\overline Q_\theta(s)$ of the critic network(s). 

While the resulting algorithm is simple, its usability is limited since it is numerically unstable in practice as we empirically show in Appendix~\ref{app:results_Qbar}. We hypothesize that the numerical instability is caused by the fact that we learn $\overline Q$ instead of $Q$. To mitigate this numerical instability, we note that Equation~\eqref{eq:bellman} is equivalent to
\begin{align}
    Q^\pi_t & = \frac1\beta\log\left(\overline Q^\pi_t\right) \nonumber \\
    & = \frac1\beta\log\left(\mathbb E_{\left(s_{t+1},a_{t+1}\right)\sim\rho_\pi} \left[ e^{\beta \left( r_t + \gamma\left( \alpha\mathcal H^\pi_{t+1} + Q^\pi_{t+1} \right)\right)}\right]\right). \label{eq:bellman_new}
\end{align}
When learning $Q$, the \gls{RHS} of Equation~\eqref{eq:bellman_new} is the target in the mean squared error loss function. Due to the discrete action space, we can compute the expectation over next actions directly, given a next state. 

Still, we need a sampling-based estimate of the \gls{RHS}: since we want to obtain a model-free algorithm that learns from single trajectories, we cannot compute the expectation over next states directly. As the expectation appears inside the logarithm, replacing the expectation by a mean over samples does not give an unbiased estimate of the \gls{RHS}. We do so nevertheless and observe that this leads to a well-performing algorithm in practice. We have only one next state $s'$ corresponding to each state $s$ that we sample from the replay buffer to compute the loss. Consequently, we remove the expectation over next states in Equation~\eqref{eq:bellman_new} and replace $s_{t+1}$ by the next state $s'$ from the replay buffer. Then, the critic loss function becomes
\footnotesize
\begin{align*}
    J_Q(\theta) & =\mathbb E_{(s,a,r,d,s')\sim D}\left[\frac12\left( Q_\theta(s,a) - \hat Q \right)^2\right] , \\
    \hat Q & = \frac1\beta\log\Bigl( \pi_\phi(s')^T \Bigr. \\
    & \hspace{1.1cm} \Bigl. \cdot\exp\left(\beta\left( r + \gamma\left( Q_{\overline\theta} \left(s'\right) - \alpha\cdot\pi_\phi(s')^T\log\pi_\phi(s') \right)\right)\right)\Bigr) .
\end{align*}
\normalsize
The computation of this target $\hat Q$ has a log-sum-exp structure, which is known to be numerically unstable when implemented naively. Thus, we rewrite it as follows (see Appendix~\ref{app:log_sum_exp} for details):
\footnotesize
\begin{align*}
    \hat Q & = r - \gamma\alpha\cdot\pi_\phi(s')^T\log\pi_\phi(s') + \gamma \max_{a'}Q_{\overline\theta}(s',a') \nonumber \\
    & \quad +\frac1\beta\log\left(\sum_{a'} \pi_\phi(a'\lvert s')\cdot \exp\left(\beta\gamma\left( Q_{\overline\theta}(s',a') - \max_{a'}Q_{\overline\theta}(s',a') \right)\right) \right) .
\end{align*}
\normalsize
With this new critic loss function, we can learn $Q$ directly, without the need to learn $\overline Q$. The resulting algorithm is identical to risk-neutral \gls{SAC} for discrete actions, except for the adapted critic loss, which gives the intended risk-sensitivity. We provide a pseudocode in Appendix~\ref{app:pseudocode}.

\section{Experiments}
We test the proposed risk-sensitive \gls{SAC} algorithm for discrete actions on an environment that abstracts several multi-stage stochastic \gls{CO} problems. In the following, we introduce the environment, discuss our benchmark algorithms, and explain how we evaluate the algorithms' performance and robustness. For details on our state encoding, NN architectures, and hyperparameters, see Appendix~\ref{app:details_experiments}.

\subsection{Environment}
We consider a discrete time horizon comprising 200 time steps per episode. Our environment is a 2D grid containing 5x5 cells, in which an agent can move around freely in the four cardinal directions, i.e., the action space at each time step is \{no move, move up, move right, move down, move left\}. If an attempted move causes the agent to exit the grid, it remains in its current cell. Each move incurs a negative reward (cost) of -1. 

Within this grid, items appear stochastically based on a spatial probability distribution at each time step. We assume that the distribution remains unknown to the agent, which can only observe the resulting data upon interacting with the environment. The items disappear after a maximum response time of ten time steps. When the agent reaches a cell containing an item before it disappears, it collects the item and should then transport it to a fixed target location (see Figure~\ref{fig:environment_illustration}). Delivering an item to the target location leads to a positive reward (revenue) of +15. The agent cannot carry more than one item at a time. This problem setting satisfies the Markov property, since the probability that an item appears in a specific cell during the current time step does neither depend on the agent's past actions nor on which items appeared in prior time steps.

Our environment is an abstraction of various classical \gls{CO} problems requiring sequential online decision-making under uncertainty, e.g., various routing and dispatching problems in the context of warehouses, mobility on demand systems, (crowd-sourced) delivery, or ambulance operations. While our environment is simple, it is also general and thus well suited to study robustness of \gls{DRL} in multi-stage stochastic \gls{CO} problems. Moreover, our environment is well suited to investigate distribution shifts, which can occur in all aforementioned applications: Figure~\ref{fig:distributions} depicts the item-generating probability distributions considered here. We train on data sampled from the gradient~1 item distribution, while testing evaluates the trained model on data sampled from all twelve distributions, simulating a wide range of distribution shifts. 

One comment on the action space is in order: our agent takes the micro-decision to which neighboring cell to move next, rather than taking the macro-decision to which cell (even if it is not a neighboring one) to move next. The latter option might appear more natural, e.g., for the decision which item to collect. However, we intentionally chose the former option, as it is more general and any macro-decision can be reconstructed through a series of micro-decisions. 

\begin{figure}
    \centering
    \includegraphics[scale=0.5]{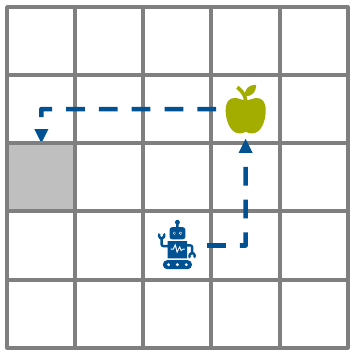}
    \caption{Illustration of the environment. The gray cell is the target location.}
    \label{fig:environment_illustration}
\end{figure}

\begin{figure}
    \centering
    \includegraphics[width=\linewidth]{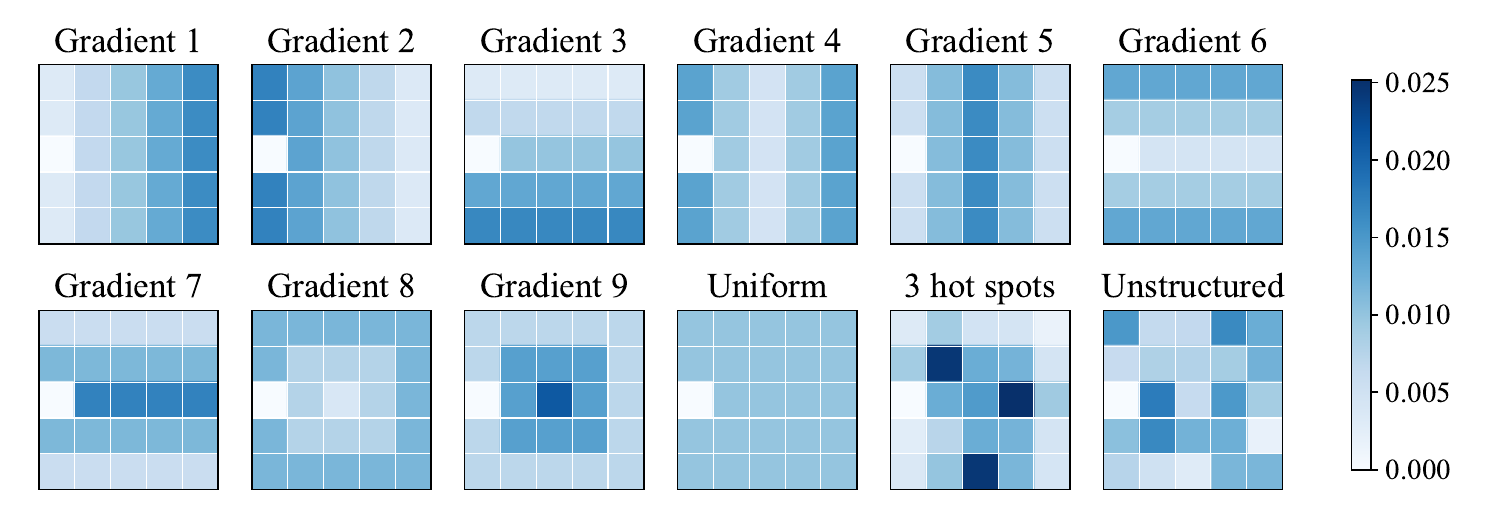}
    \caption{Per-time-step probability that an item appears in the respective cell, for twelve different item distributions.}
    \label{fig:distributions}
\end{figure}

\subsection{Benchmarks} \label{sec:benchmarks}
We compare our risk-sensitive \gls{SAC} algorithm against two benchmarks: manipulating the training data and entropy regularization. Moreover, we compare all three approaches to improve robustness against risk-neutral \gls{SAC} for discrete actions, which is our non-robust baseline algorithm. We further hypothesized that L2 regularization of the \gls{NNs}' parameters could improve robustness, as it is frequently used in (supervised) machine learning to use a complex model but reduce overfitting, and learning a policy that is less tailored to the training data might help robustness against distribution shifts. However, we did not find evidence for this hypothesis. We omit details on this important negative result in the main body of the paper, but provide them in Appendix~\ref{app:L2_reg}.

\paragraph{Manipulating the training data.} In supervised learning, manipulating the training data has been used successfully to achieve robustness, see \cite{Goodfellow2015, Tramer2018, Sinha2018}. We transfer this approach to \gls{RL} by inserting noise into the training process. We expect the policy to learn to perform well under this noise, generalize better, and thus be more robust against disturbances during testing. 

For distribution shifts in our environment, this means to replace a part of the item locations in the training data by item locations sampled from a different spatial distribution. Since we assume the distribution shift to be unknown a priori, we use the uniformly random distribution to sample the manipulated item locations. Specifically, we replace each item in the training data with a certain probability, for which we test a large range of different values. 

For a fair comparison, we train risk-neutral \gls{SAC} on the manipulated training data and test the resulting policy on the not-manipulated test data for all distributions in Figure~\ref{fig:distributions}. This includes gradient~1, to assess how manipulating the training data changes the trained policy's test performance without distribution shifts.

\paragraph{Entropy regularization.} The use of entropy regularization can increase the robustness of \gls{RL} against disturbances in the environment, as shown theoretically in \cite{Eysenbach2022} and empirically in \cite{Haarnoja2018a, Haarnoja2018b, Eysenbach2022}. Risk-neutral \gls{SAC} already incorporates entropy regularization. Its intensity can be controlled via the entropy coefficient hyperparameter $\alpha$. For our non-robust baseline algorithm, for which we focus on maximizing performance on the gradient~1 distribution, we found a scheduled $\alpha$ to be effective: we use a tuned $\alpha>0$ at the beginning of training for effective exploration, then set it to zero after a tuned number of iterations, and continue training with $\alpha=0$ until convergence, obtaining the final policy without entropy regularization. We use the same schedule for $\alpha$ in our risk-sensitive \gls{SAC} algorithm. When we use entropy regularization to improve the robustness of the risk-neutral algorithm, we keep the tuned $\alpha$-value at the beginning of training, but then reset it to a positive value for the remainder of training until convergence. We conduct experiments for a large range of different values for this final $\alpha$. 

\subsection{Performance evaluation} \label{sec:performance_evaluation}
We measure an algorithm's performance in our environment in terms of its percentage improvement over a greedy baseline algorithm, as greedy algorithms typically perform well in related \gls{CO} problems \cite{Enders2023}. Advanced algorithms typically outperform greedy algorithms by only a few percentage points. Still, the effort to develop such algorithms is meaningful, since they are typically deployed at a large scale in practice, such that the small percentage improvements translate into large absolute gains \cite{SadeghiEshkevari2022}. 

We implement the following greedy algorithm: if the agent has collected but not yet delivered an item, it moves to the target location on the shortest route. If the agent has no item on board, but there is an item available that can be reached before it disappears and will lead to a positive profit (accounting for revenue and cost for moving to the item and then to the target location), the agent moves towards the item's location on the shortest route. If there are multiple such items available, the agent moves towards the item that will lead to the highest profit. If there is no such item available, the agent does not move, i.e., it stays at the target location. When the agent has started to move towards an item and then a more profitable item appears, the agent changes direction and moves towards the more profitable item. 

This greedy algorithm is not tailored to a specific item distribution, i.e., it is robust against distribution shifts. Thus, when we report performance improvement over the greedy algorithm, we report performance improvement over a robust baseline, which makes the greedy algorithm particularly well suited for our purposes. 

When we report the performance of a \gls{DRL} algorithm under a distribution shift, i.e., the performance on test data from a distribution different than the training distribution, we report its performance gain over greedy relative to the performance gain of the policy obtained by training a risk-neutral \gls{SAC} agent on data from the respective shifted distribution. Thereby, we view the performance of a \gls{DRL} agent trained on the ``true'' distribution as an upper bound on the performance of a robust or non-robust \gls{DRL} agent trained on the ``wrong'' distribution. 

We use 1,000 episodes of sampled data for each of the item distributions, which we split into 800 training, 100 validation, and 100 testing episodes. The test performance is the non-discounted cumulative reward per episode, averaged over the 100 testing episodes. We repeat every training run with three different random seeds and select the model with the highest validation reward across seeds for testing. 

To reduce the impact of the random sampling process on the reported results for our benchmark based on manipulated training data, we repeat the training data manipulation process three times with three different random seeds. Then, we repeat every experiment for each of the three manipulated data sets. Finally, we report the test performance averaged over the three policies trained on the three different data sets.

\section{Results and discussion}

In the following, we briefly discuss the convergence behavior of our risk-sensitive algorithm. Then, we analyze the performance of our algorithm and the two benchmarks in detail. 

\subsection{Convergence analysis}

Figure~\ref{fig:convergence} shows the validation reward over the course of training for our risk-sensitive algorithm with $\beta=-1$ in comparison to risk-neutral \gls{SAC}. The training curves illustrate that for small absolute $\beta$-values, our risk-sensitive algorithm shows stable convergence behavior. For larger absolute values of $\beta$, our algorithm still exhibits stable convergence to a good policy for specific random seeds, but not across all tested random seeds. Moreover, for all tested values of $\beta$, our risk-sensitive algorithm requires a similar number of samples and a similar computational time as risk-neutral \gls{SAC}.

\begin{figure}[h]
    \centering
    \includegraphics[width=\linewidth]{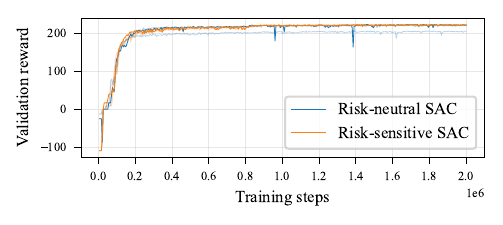}
    \caption{Convergence behavior of our risk-sensitive algorithm with $\beta=-1$ compared to risk-neutral \gls{SAC}. For each algorithm, we show the training curves for three different random seeds. The non-transparent lines correspond to the best-performing seed, the transparent ones to the other seeds.}
    \label{fig:convergence}
\end{figure}

\subsection{Performance analysis}

\begin{figure*}
    \centering
    \includegraphics[width=\linewidth]{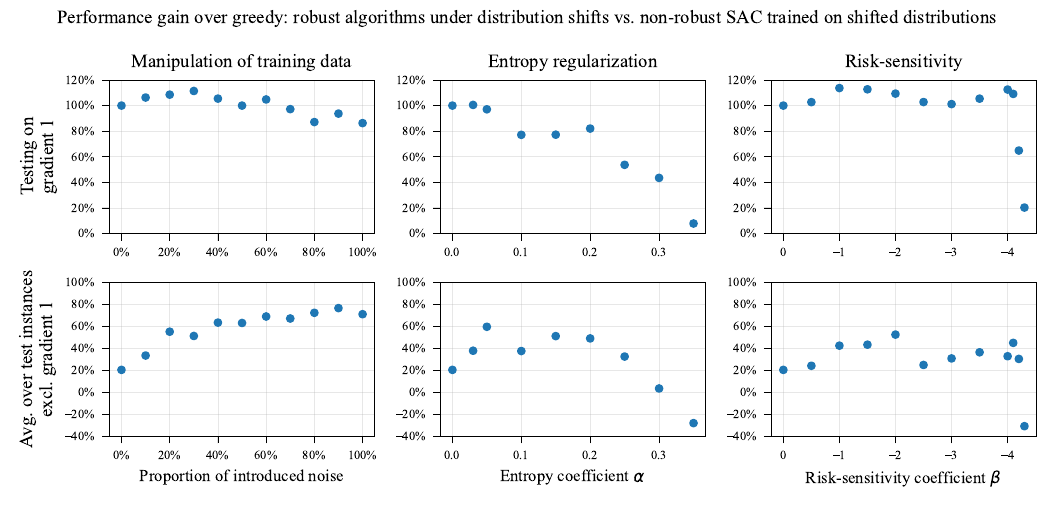}
    \caption{Performance of the three approaches to improve robustness. The top row shows their performance on the test data for the training distribution. The bottom row shows the average over all other distributions, i.e., the performance under distribution shifts. The left-most data point in each plot shows the results for the non-robust discrete \gls{SAC} algorithm. We report all results relative to \gls{SAC}'s performance when trained on the shifted, i.e., the true distribution as explained in Section~\ref{sec:performance_evaluation}.}
    \label{fig:results}
\end{figure*}

Figure~\ref{fig:results} shows the performance of the three approaches to improve robustness, both on the test data for the training distribution and under distribution shifts. As theoretically expected, the non-robust \gls{SAC} algorithm performs worse under distribution shifts than when it is trained on data from the true distribution. Specifically, the performance gain over greedy under distribution shifts is only 20\% of the achievable performance gain. This illustrates that our environment and the considered item distributions are suitable to experimentally investigate the robustness of \gls{DRL} algorithms under distribution shifts.

Our risk-sensitive algorithm improves the performance under distribution shifts compared to risk-neutral \gls{SAC}, as it reaches up to 53\% of the upper bound's performance improvement over greedy. Up to $\beta=-2$, an increasing absolute value of $\beta$ and therefore increasing risk-aversion improves the performance under distribution shifts. As the absolute value of $\beta$ becomes even larger, the risk-sensitive algorithm still outperforms the risk-neutral algorithm under distribution shifts, but performance decreases again and the performance pattern becomes slightly unstable. This instability is in line with the previously described worsening convergence behavior as the absolute value of $\beta$ increases. These empirical observations align with our theoretical results in Section~\ref{sec:methodology}, which assume that $\beta$ is close to zero. Finally, as the absolute value of $\beta$ becomes too large, performance collapses, as the agent becomes too risk-averse and the learning process too unstable to learn a good policy. 

Except for too large absolute $\beta$-values, the risk-averse policy consistently outperforms the risk-neutral policy on the training distribution, improving the risk-neutral policy's performance gain over greedy up to 114 percentage points. This result might be surprising at first sight, as one might expect a consistency-robustness tradeoff. However, it is in line with experimental results in the literature \cite{Ma2020, Noorani2021}, where risk-averse \gls{DRL} algorithms outperform their risk-neutral counterparts on the training environments. The authors of \cite{Noorani2021} explain this observation by a variance reduction due to the risk-averse objective, which in turn helps to converge to a better policy for the training environment. Based on Equation~\eqref{eq:taylor_expansion}, our risk-sensitive objective can be interpreted as the variance-regularized expected return, penalizing high variance for negative $\beta$-values. This explains our risk-averse algorithm's performance improvement on the gradient~1 distribution. 

The manipulation of the training data also improves the performance under distribution shifts, to about 60-80\% of the upper bound's performance improvement over greedy when we sample at least 40\% of item locations from the uniform distribution. Simultaneously, the performance on the gradient~1 distribution increases for small proportions of introduced noise, but decreases as this proportion becomes larger. Consequently, manipulating the training data leads to better robustness results than our risk-sensitive algorithm. Which approach is favorable depends on how one weighs consistency versus robustness, as our risk-sensitive algorithm achieves better results on the training distribution. Besides, the manipulation of the training data requires the ability to change the training environment deliberately, typically in a simulator, and the domain expertise how to manipulate the training data effectively for the respective problem setting. We can use our risk-sensitive algorithm as well as entropy regularization even when these requirements are not fulfilled, such that these are more generally applicable.

With entropy regularization, the performance under distribution shifts improves to 60\% of the upper bound's performance improvement over greedy as the entropy regularization coefficient $\alpha$ increases. For further increasing $\alpha$, the performance under distribution shifts decreases again because the over-regularization hurts performance. On the training distribution, the performance decreases as $\alpha$ increases. Consequently, entropy regularization leads to slightly better robustness than our risk-sensitive algorithm, at the price of lower performance on the training distribution. Note that we also combined our risk-sensitive algorithm with entropy regularization, but found that this does not further improve performance, see Appendix~\ref{app:combination}.

Figure~\ref{fig:tradeoff} compares the consistency-robustness tradeoff for our risk-sensitive algorithm and entropy regularization. Here, we consider the weighted average of the performance on gradient~1 and the performance under distribution shifts to evaluate these two approaches to improve the robustness of \gls{DRL} based on a single metric. When we assign at least 37\% weight to the performance on the training distribution, our risk-sensitive algorithm outperforms entropy regularization. Thus, assuming that good performance on the training distribution is not only a secondary objective with less than 37\% weight, our algorithm is superior to entropy regularization in our experiments. Nevertheless, the final decision which approach is best suited to ensure robustness depends on the needs of the considered problem setting. 

\begin{figure}
    \centering
    \includegraphics[width=\linewidth]{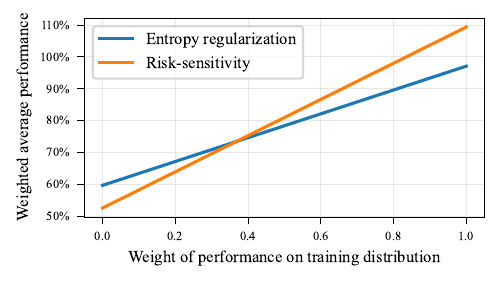}
    \caption{Consistency-robustness tradeoff for entropy regularization with $\alpha=0.05$ and our risk-sensitive algorithm with $\beta=-2$. Using the performance under distribution shifts vs. non-robust \gls{SAC} trained on the shifted distributions as the performance metric, we show the weighted average of the performance on the training distribution and the performance under distribution shifts.}
    \label{fig:tradeoff}
\end{figure}

\section{Conclusion}

We present discrete \gls{SAC} for the entropic risk measure, which is the first model-free risk-sensitive \gls{DRL} algorithm for discrete actions that builds on the state-of-the-art in risk-neutral \gls{DRL}. Specifically, we derive a version of the Bellman equation for $Q$-values for the entropic risk measure. We establish a corresponding policy improvement result and infer a practical off-policy algorithm that learns from single trajectories. Our algorithm allows to control its risk-sensitivity via a hyperparameter, and implementing our algorithm requires only a small modification relative to risk-neutral \gls{SAC}, such that it is easily applicable in practice. We conduct experiments within an environment representing typical contextual multi-stage stochastic \gls{CO} problems from the operations research domain. Thereby, we demonstrate that our risk-sensitive algorithm significantly improves robustness against distribution shifts compared to risk-neutral \gls{SAC}. Simultaneously, it improves the performance on the training distribution. We compare our risk-sensitive algorithm to (i) manipulation of the training data and (ii) entropy regularization, showing that our algorithm is superior to these benchmarks due to its more general applicability and a better consistency-robustness tradeoff. Our study is the first structured analysis of the robustness of \gls{RL} under distribution shifts in the realm of contextual multi-stage stochastic \gls{CO} problems. Moreover, we note that the presented algorithm is relevant beyond this application domain and may be used for any \gls{MDP} with discrete actions where risk-sensitivity or robustness against environment perturbations is relevant.

In future work, we will extend this research in the following directions: we will investigate the performance of our algorithm in a more complex, large-scale version of our environment with multiple agents. Furthermore, we will develop a version of our risk-sensitive \gls{SAC} algorithm for continuous actions. Finally, we will implement and test a risk-sensitive \gls{DQN} algorithm for the entropic risk measure based on the derived Bellman equation.

\newpage

\bibliographystyle{ACM-Reference-Format}
{\footnotesize \bibliography{references}}

\newpage

\appendix

\section{Details on derivations in Section~\ref{sec:methodology}}

\subsection{Approximations} \label{app:approximations}
Approximation (i): for $\beta$ close to zero and a real-valued random variable $X$, we get
\begin{equation*}
    \mathbb E\left[e^{\beta X}\right] = \exp\left(\beta\, \mathbb E\left[X\right] + \frac{\beta^2}{2}\mathrm{Var}\left[X\right] + \mathcal O\left(\beta^3\right)\right) \approx e^{\beta\, \mathbb E\left[X\right]}.
\end{equation*}
Approximation (ii): for a function $f$, it holds that
\begin{equation*}
    \mathbb E\left[f(X)\right] \approx f\left(\mathbb E[X]\right) + \frac12 f''\left(\mathbb E[X]\right)\cdot\mathrm{Var}[X] .
\end{equation*}
With $f(x)=x^\gamma$, $f'(x)=\gamma x^{\gamma-1}$, $f''(x)=\gamma(\gamma-1)x^{\gamma-2}$, we get
\begin{equation*}
    \mathbb E\left[X^\gamma\right] \approx \left(\mathbb E[X]\right)^\gamma + \frac12 \gamma(\underbrace{\gamma-1}_{\approx 0})\left(\mathbb E[X]\right)^{\gamma-2}\mathrm{Var}[X] \approx \left(\mathbb E[X]\right)^\gamma
\end{equation*}
for $\gamma$ close to one.

\subsection{Policy improvement proof} \label{app:proof_policy_improvement}
The proof of Proposition~\ref{prop:policy_improvement} follows the same arguments as the policy improvement proof in \cite{Haarnoja2018}. Since we can always choose $\pi_\text{new} = \pi_\text{old}$, it holds that
\begin{equation*}
    J_{\pi_\text{old}}\left(\pi_\text{new}\left(\cdot\lvert s_t\right)\right) \leq J_{\pi_\text{old}}\left(\pi_\text{old}\left(\cdot\lvert s_t\right)\right) .
\end{equation*}
Consequently,
\begin{align*}
    &\mathbb E_{a_t\sim\pi_\text{new}} \left[\log\pi_\text{new}\left(a_t\lvert s_t\right) - \frac{1}{\alpha}Q^{\pi_\text{old}}_t+\log Z^{\pi_\text{old}}\left(s_t\right)\right] \nonumber \\
    &\leq \mathbb E_{a_t\sim\pi_\text{old}} \left[\log\pi_\text{old}\left(a_t\lvert s_t\right) - \frac{1}{\alpha}Q^{\pi_\text{old}}_t+\log Z^{\pi_\text{old}}\left(s_t\right)\right] .
\end{align*}
Since $Z^{\pi_\text{old}}$ depends only on the state, not the action, this yields
\begin{align*}
    &\mathbb E_{a_t\sim\pi_\text{new}} \left[Q^{\pi_\text{old}}_t-\alpha\log\pi_\text{new}\left(a_t\lvert s_t\right)\right] \nonumber \\
    &\geq \mathbb E_{a_t\sim\pi_\text{old}} \left[Q^{\pi_\text{old}}_t-\alpha\log\pi_\text{old}\left(a_t\lvert s_t\right)\right] .
\end{align*}
We assume $\beta<0$ in the following and note later that the case $\beta>0$ works similarly. For $\beta<0$, we get
\begin{align*}
    & \beta\cdot \mathbb E_{a_t\sim\pi_\text{new}} \left[Q^{\pi_\text{old}}_t-\alpha\log\pi_\text{new}\left(a_t\lvert s_t\right)\right] \nonumber \\
    & \leq \beta\cdot \mathbb E_{a_t\sim\pi_\text{old}} \left[Q^{\pi_\text{old}}_t-\alpha\log\pi_\text{old}\left(a_t\lvert s_t\right)\right] \nonumber \\[0.3\baselineskip]
    \iff & \beta\cdot\left(\alpha\mathcal H^{\pi_\text{new}}_t + \mathbb E_{a_t\sim\pi_\text{new}} \left[Q^{\pi_\text{old}}_t\right]\right) \nonumber \\
    & \leq \beta\cdot\left(\alpha\mathcal H^{\pi_\text{old}}_t + \mathbb E_{a_t\sim\pi_\text{old}} \left[Q^{\pi_\text{old}}_t\right]\right) \nonumber \\[0.3\baselineskip]
    \iff & \exp\left(\beta\alpha\mathcal H^{\pi_\text{new}}_t + \beta\cdot\mathbb E_{a_t\sim\pi_\text{new}} \left[Q^{\pi_\text{old}}_t\right]\right) \nonumber \\
    & \leq \exp\left(\beta\alpha\mathcal H^{\pi_\text{old}}_t + \beta\cdot\mathbb E_{a_t\sim\pi_\text{old}} \left[Q^{\pi_\text{old}}_t\right]\right) \nonumber \\[0.3\baselineskip]
    \iff & e^{\beta\alpha\mathcal H^{\pi_\text{new}}_t}\cdot\mathbb E_{a_t\sim\pi_\text{new}} \left[e^{\beta Q^{\pi_\text{old}}_t}\right] \nonumber \\
    & \leq e^{\beta\alpha\mathcal H^{\pi_\text{old}}_t}\cdot\mathbb E_{a_t\sim\pi_\text{old}} \left[e^{\beta Q^{\pi_\text{old}}_t}\right] \nonumber \\[0.3\baselineskip]
    \iff & e^{\beta\alpha\mathcal H^{\pi_\text{new}}_t}\cdot\mathbb E_{a_t\sim\pi_\text{new}} \left[\overline Q^{\pi_\text{old}}_t\right] \nonumber \\
    & \leq e^{\beta\alpha\mathcal H^{\pi_\text{old}}_t}\cdot\mathbb E_{a_t\sim\pi_\text{old}} \left[\overline Q^{\pi_\text{old}}_t\right],
\end{align*}
where we used Approximation (i). Repeated application of the Bellman equation derived before and this inequality yields
\footnotesize
\begin{align*}
    \overline Q^{\pi_\text{old}}_t &= e^{\beta r_t}\cdot \mathbb E_{s_{t+1}\sim p} \left[\mathbb E_{a_{t+1}\sim\pi_\text{old}}\left[ \overline Q^{\pi_\text{old}}_{t+1}\right] \cdot e^{\beta\alpha\mathcal H^{\pi_\text{old}}_{t+1}} \right]^\gamma \nonumber \\
    & \geq e^{\beta r_t}\cdot \mathbb E_{s_{t+1}\sim p} \left[\mathbb E_{a_{t+1}\sim\pi_\text{new}}\left[ \overline Q^{\pi_\text{old}}_{t+1}\right] \cdot e^{\beta\alpha\mathcal H^{\pi_\text{new}}_{t+1}} \right]^\gamma \nonumber \\
    & \;\;\vdots \nonumber \\
    &= e^{\beta r_t}\cdot \mathbb E_{\left(s_{t+1},...\right)\sim\rho_{\pi_\text{new}}} \left[ \exp\left( \beta\cdot\sum_{l=1}^\infty \gamma^l \left( r_{t+l} + \alpha\mathcal H^{\pi_\text{new}}_{t+l} \right) \right) \right] \nonumber \\
    &= \overline Q^{\pi_\text{new}}_t .
\end{align*}
\normalsize
Thus, we obtain
\begin{equation*}
    Q^{\pi_\text{old}}_t = \frac{1}{\beta}\log\overline Q^{\pi_\text{old}}_t \leq \frac{1}{\beta}\log\overline Q^{\pi_\text{new}}_t = Q^{\pi_\text{new}}_t,
\end{equation*}
which proves the policy improvement for $\beta<0$. The proof for $\beta>0$ works analogously (twice, the $\leq$ or $\geq$-sign does not turn).

\section{Results when learning $\overline Q$} \label{app:results_Qbar}

Figure~\ref{fig:Qbar} depicts the convergence behavior of the risk-sensitive algorithm that is based on learning $\overline Q$ as opposed to $Q$. For negative values of $\beta$, the plot looks similar to the one for $\beta=0.01$, i.e., the algorithm converges to zero rewards (we do not include the results for $\beta<0$ since they overlap too much with the results for $\beta=0.01$). While the training curve is reasonable for $\beta=0.1$, we only obtain a ``do not move'' policy with zero rewards as $\beta\rightarrow0$ and for $\beta<0$. 

\begin{figure}[h]
    \centering
    \includegraphics[width=\linewidth]{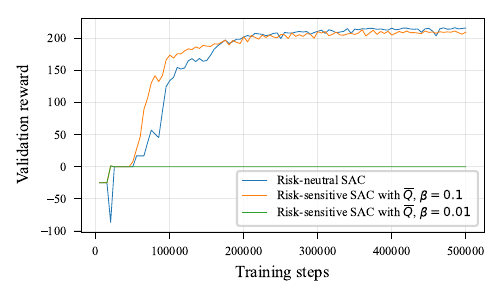}
    \caption{Convergence behavior of the risk-sensitive algorithm based on learning $\overline Q$ compared to risk-neutral \gls{SAC}.}
    \label{fig:Qbar}
\end{figure}

\section{Reformulating target in critic loss for numerically stable log-sum-exp computation} \label{app:log_sum_exp}

The critic loss function reads
\footnotesize
\begin{align*}
    J_Q(\theta) & =\mathbb E_{(s,a,r,d,s')\sim D}\left[\frac12\left( Q_\theta(s,a) - \hat Q \right)^2\right] , \\
    \hat Q & = \frac1\beta\log\Bigl( \pi_\phi(s')^T \Bigr. \\
    & \hspace{1.1cm} \Bigl. \cdot\exp\left(\beta\left( r + \gamma\left( Q_{\overline\theta} \left(s'\right) - \alpha\cdot\pi_\phi(s')^T\log\pi_\phi(s') \right)\right)\right)\Bigr) .
\end{align*}
\normalsize
The computation of this target $\hat Q$ has a log-sum-exp structure, which is known to be numerically unstable when implemented naively. Thus, we rewrite it using the following trick, based on Section~5.2 of \cite{Nesterov2005}: for some variable $x_a$ that depends on the action $a$, we define $y_a=x_a-\overline x$, with $\overline x$ independent of $a$, and get
\begin{equation*}
    \frac1\beta \log\left(\sum_a \pi(a\lvert s)\cdot e^{\beta x_a}\right) = \overline x + \frac1\beta \log\left(\sum_a \pi(a\lvert s)\cdot e^{\beta y_a}\right).
\end{equation*}
We choose $\overline x = \max_a x_a$ and use this identity to rewrite $\hat Q$ as follows:
\footnotesize
\begin{align*}
    \hat Q & = r - \gamma\alpha\cdot\pi_\phi(s')^T\log\pi_\phi(s') + \gamma \max_{a'}Q_{\overline\theta}(s',a') \nonumber \\
    & \quad +\frac1\beta\log\left(\sum_{a'} \pi_\phi(a'\lvert s')\cdot \exp\left(\beta\gamma\left( Q_{\overline\theta}(s',a') - \max_{a'}Q_{\overline\theta}(s',a') \right)\right) \right) .
\end{align*}
\normalsize

\section{Pseudocode} \label{app:pseudocode}

\begin{algorithm}
    \caption{Risk-sensitive discrete Soft Actor-Critic}
    \label{alg:pseudocode}
    \begin{algorithmic}[1] 
        \State Initialize NN parameters $\phi, \theta_i, \overline\theta_i$ for $i\in\{1,2\}$
        \State Initialize an empty replay buffer $D$
        \State Initialize the environment, observe $s$
        \For{each iteration}
            \State $a\sim\pi_\phi(a\lvert s)$
            \State Execute $a$ in the environment and observe $r$, $s'$
            \State $D \gets D\cup\{(s,a,r,s')\}$
            \State $s \gets s'$
            \State Sample a batch of transitions from $D$
            \State $\theta_i \gets \theta_i - \lambda\cdot \nabla_{\theta_i} J_Q(\theta_i)$ for $i\in\{1,2\}$
            \State $\overline\theta_i \gets (1-\tau) \cdot \overline\theta_i + \tau \cdot \theta_i$ for $i\in\{1,2\}$
            \State $\phi \gets \phi - \lambda\cdot \nabla_\phi J_\pi(\phi)$
        \EndFor
    \end{algorithmic}
\end{algorithm}

We provide the pseudocode for our risk-sensitive discrete \gls{SAC} algorithm in Algorithm~\ref{alg:pseudocode}. Here, $\theta_i$ for $i\in\{1,2\}$ refers to the parameters of the two critics which we train concurrently as explained in Section~\ref{sec:risk_neutral_SAC}. Besides, $\lambda$ is the learning rate and $\tau$ is the smoothing factor for the exponential moving average used to update the target critic parameters.

\section{Details on experiments} \label{app:details_experiments}

In the following, we firstly provide details on our state encoding and NN architectures. Secondly, we report the hyperparameters used in our experiments. 

\subsection{State encoding and neural networks}

The state of our environment can be naturally encoded as a 5x5 image with one channel per type of element contained in the system: the target location, the agent, and the items. The target location channel has only zero entries except for the target location, which we represent by a one. The agent channel has only zero entries except for the agent's current location. At this location, the entry is 0.5 if the agent has picked up but not yet delivered an item. Otherwise, this entry is one. The items channel has only zero entries except for locations with items that have not been picked up and did not disappear because the maximum response time elapsed. Such an entry is equal to the number of time steps remaining until the respective item will disappear, normalized by the maximum response time. 

We use a combination of convolutional and fully connected layers in our \gls{NNs}. Except for the output activation function, all \gls{NNs} have the same architecture. We use the following sequence of layers:
\begin{itemize}
    \setlength\itemsep{-0.1em}
    \item Convolutional layer 1: 2D convolutional layer with 32 filters of size 3x3 with stride one, same padding and ReLU activation
    \item Convolutional layer 2: 2D convolutional layer with 64 filters of size 2x2 with stride one, same padding and ReLU activation
    \item Convolutional layer 3: 2D convolutional layer with 64 filters of size 2x2 with stride one, same padding and ReLU activation
    \item Flatten layer
    \item Fully connected layer 1: fully connected layer with 256 units and ReLU activation
    \item Fully connected layer 2: fully connected layer with 256 units and ReLU activation
    \item Output layer: fully connected layer with 5 units (for the 5 possible actions) and softmax activation for the actor network and linear activation for the critic networks, respectively
\end{itemize}
We use L2 regularization for the NN parameters with a regularization coefficient of $10^{-4}$.

\subsection{Hyperparameters}

We train for 2 million steps, update the network parameters every 20 steps, and test the performance of the current policy on the validation data every 5,000 steps. During the first 20,000 steps, we collect experience with a random policy and do not update the network parameters. 

We set the discount factor to 0.99. We sample batches of size 512 from a replay buffer with maximum size 200,000. When we sample transitions from the replay buffer, we normalize the sampled rewards by dividing them by the standard deviation of all rewards currently stored in the replay buffer. For the critic loss, we use the Huber loss with a delta value of 2 instead of the squared error. Moreover, we use gradient clipping with a clipping ratio of 10 for actor and critic gradients. To update the NN parameters, we use the Adam optimizer with a learning rate of $3\cdot10^{-4}$. For the update of the target critic parameters, we use an exponential moving average with smoothing factor $5\cdot10^{-3}$. 

We tune the entropy coefficient individually per experiment and use values between 0.1 and 0.3 across our experiments. After 800,000 training steps, we set it to zero as explained in Section~\ref{sec:benchmarks}.

\section{L2 regularization} \label{app:L2_reg}

The non-robust discrete \gls{SAC} algorithm already uses L2 regularization with an L2 regularization coefficient of 0.0001, since we find through hyperparameter tuning that this maximizes performance on the validation data from the gradient~1 distribution. To investigate if L2 regularization improves the algorithm's robustness against distribution shifts, we increase the L2 regularization coefficient. Figure~\ref{fig:L2_reg} shows the experimental results. We also tried even larger values for the L2 regularization coefficient than the ones depicted here, but find that they lead to learning a ``do not move'' policy with zero rewards due to over-regularization. Based on the results in Figure~\ref{fig:L2_reg}, we conclude that L2 regularization does not improve the robustness of \gls{SAC} against distribution shifts within our environment. 

\begin{figure}[h]
    \centering
    \includegraphics[width=\linewidth]{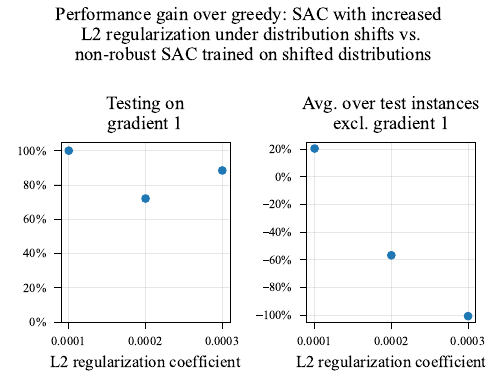}
    \caption{Performance of L2 regularization to improve robustness. The left plot shows the performance on the test data for the training distribution. The right plot shows the average over all other distributions, i.e., the performance under distribution shifts. The data points for an L2 regularization coefficient of 0.0001 show the results for the non-robust discrete \gls{SAC} algorithm. We report all results relative to \gls{SAC}'s performance when trained on the shifted, i.e., the true distribution as explained in Section~\ref{sec:performance_evaluation}.}
    \label{fig:L2_reg}
\end{figure}

\section{Combination of risk-sensitivity and entropy regularization} \label{app:combination}

We combine our risk-sensitive algorithm with entropy regularization to evaluate if this combination improves upon the two approaches' individual performance. Figure~\ref{fig:combination} shows the experimental results. None of the tested configurations for $\beta$ and $\alpha$ improves upon the best results for pure risk-sensitivity or pure entropy regularization in Figure~\ref{fig:results}.

\begin{figure}[h]
    \centering
    \includegraphics[width=\linewidth]{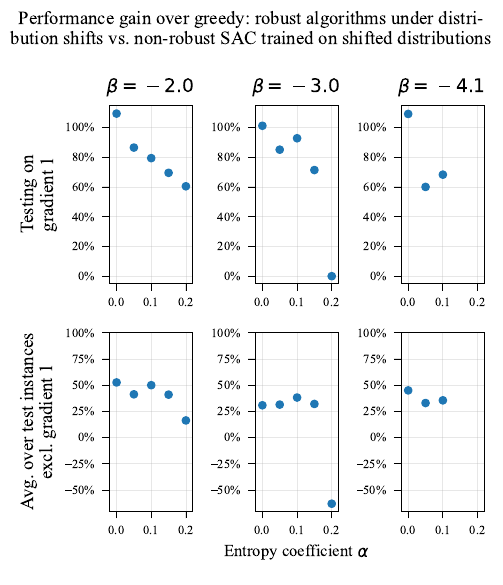}
    \caption{Performance of combining our risk-sensitive algorithm with entropy regularization to further improve robustness. The top row shows the performance on the test data for the training distribution. The bottom row shows the average over all other distributions, i.e., the performance under distribution shifts. The left-most data point in each plot shows the results for our risk-sensitive algorithm without entropy regularization. There are no data points for $\alpha\in\{0.15,0.2\}$ in the right-most plots, because the corresponding training runs converge to a validation reward that is substantially worse than the greedy performance on the gradient~1 distribution. We report all results relative to \gls{SAC}'s performance when trained on the shifted, i.e., the true distribution as explained in Section~\ref{sec:performance_evaluation}.}
    \label{fig:combination}
\end{figure}

\end{document}